\documentclass[letterpaper, 10 pt, conference]{ieeeconf}
\IEEEoverridecommandlockouts 
\usepackage[nocompress]{cite}

\usepackage{amsmath, amsfonts, mathtools}

\usepackage{amsthm}
\usepackage{theoremref}
\usepackage[english]{babel}
\newtheorem{theorem}{Theorem}
\newtheorem{lemma}[theorem]{Lemma}

\newtheorem{prop}[theorem]{Proposition}
\usepackage{verbatim}
\usepackage{romannum} 
\AtBeginDocument{\pagenumbering{arabic}} 

\usepackage{graphicx}
\usepackage{subcaption}
\usepackage[dvipsnames]{xcolor}
\usepackage[normalem]{ulem} 

\usepackage{textcomp}

\usepackage{algorithm}
\usepackage[noend]{algpseudocode}

\usepackage{array}
\usepackage{balance}
\usepackage{adjustbox}
\usepackage{multicol}
\usepackage{multirow}
\usepackage{makecell}

\usepackage{stfloats}

\usepackage{url}
\makeatletter
\let\NAT@parse\undefined
\makeatother
\usepackage[colorlinks=true]{hyperref}

\def\BibTeX{{\rm B\kern-.05em{\sc i\kern-.025em b}\kern-.08em
    T\kern-.1667em\lower.7ex\hbox{E}\kern-.125emX}}

\newcommand{\bx}{\boldsymbol{x}}
\newcommand{\bmu}{\boldsymbol{\mu}}
\newcommand{\bbeta}{\boldsymbol{\beta}}
\newcommand{\balpha}{\boldsymbol{\alpha}}

\usepackage{background}
\backgroundsetup{
    placement=top,
    contents={\parbox{1.2\textwidth}{
    This paper has been accepted for publication in the IEEE Robotics and Automation Letters.\\
    Please cite the paper as: 
    Chen, B. S., Lin, Y. K., Chen, J. Y., Huang, C. W., Chern, J. L., Sun, C. C. (2024). \textit{FracGM: A Fast Fractional Programming Technique for Geman-McClure Robust Estimator.} IEEE Robotics and Automation Letters, 9(12), 11666-11673. DOI: 10.1109/LRA.2024.3495372.
    }},
    color=NavyBlue,
    opacity=1,
    scale=0.84,
    vshift=-32,
}

\begin{document}

\title{
FracGM: A Fast Fractional Programming Technique for Geman-McClure Robust Estimator}

\author{Bang-Shien Chen$^{1}$, Yu-Kai Lin$^{2}$, Jian-Yu Chen$^{3}$, Chih-Wei Huang$^{3}$, Jann-Long Chern$^{1}$, and Ching-Cherng Sun$^{4}$
\thanks{
This work is supported by National Science and Technology Council, Taiwan (NSTC 113-2218-E-008-011-MBK, NSTC 110-2115-M-003-019-MY3).
\textit{(Corresponding author: Chih-Wei Huang.)}
}
\thanks{$^{1}$Bang-Shien Chen and Jann-Long Chern are with the Department of Mathematics, National Taiwan Normal University, Taipei 106209, Taiwan. (e-mail: dgbshien@gmail.com; chern@math.ntnu.edu.tw).}
\thanks{$^{2}$Yu-Kai Lin is with the MediaTek Inc., Hsinchu 30078, Taiwan. (e-mail: stephen.lin@mediatek.com).}
\thanks{$^{3}$Jian-Yu Chen and Chih-Wei Huang are with the Department of Communication Engineering, National Central University, Taoyuan 320317, Taiwan. (e-mail: 111523039@cc.ncu.edu.tw; cwhuang@ce.ncu.edu.tw).}
\thanks{$^{4}$Ching-Cherng Sun is with the Department of Optics and Photonics, National Central University,
Taoyuan 320317, Taiwan. (e-mail: ccsun@dop.ncu.edu.tw).}
\thanks{Digital Object Identifier (DOI): see top of this page.}
}

\maketitle
\thispagestyle{empty}
\pagestyle{empty}

\begin{abstract}
Robust estimation is essential in computer vision, robotics, and navigation, aiming to minimize the impact of outlier measurements for improved accuracy. We present a fast algorithm for Geman-McClure robust estimation, FracGM, leveraging fractional programming techniques. This solver reformulates the original non-convex fractional problem to a convex dual problem and a linear equation system, iteratively solving them in an alternating optimization pattern. Compared to graduated non-convexity approaches, this strategy exhibits a faster convergence rate and better outlier rejection capability. In addition, the global optimality of the proposed solver can be guaranteed under given conditions. We demonstrate the proposed FracGM solver with Wahba's rotation problem and 3-D point-cloud registration along with relaxation pre-processing and projection post-processing. Compared to state-of-the-art algorithms, when the outlier rates increase from 20\% to 80\%, FracGM shows 53\% and 88\% lower rotation and translation increases. In real-world scenarios, FracGM achieves better results in 13 out of 18 outcomes, while having a 19.43\% improvement in the computation time.
\end{abstract}

\section{Introduction}
Robust estimation is an essential technique in real-world applications, including computer vision, robotics, and navigation. Its main objective is to mitigate the impact of outlier measurements, thereby improving the accuracy of the estimation. Recent studies have used various types of robust functions in different applications, such as rotation averaging~\cite{CG13, CG17, HAT11, WS13}, point cloud registration~\cite{DYD21, FA03, LDH19, MZT13, YSC21}, pose graph optimization~\cite{CC18, LFP13, SP12}, and satellite navigation~\cite{JMM23, JS24}.

\begin{figure}[htbp]
    \centering
    \includegraphics[scale=0.192]{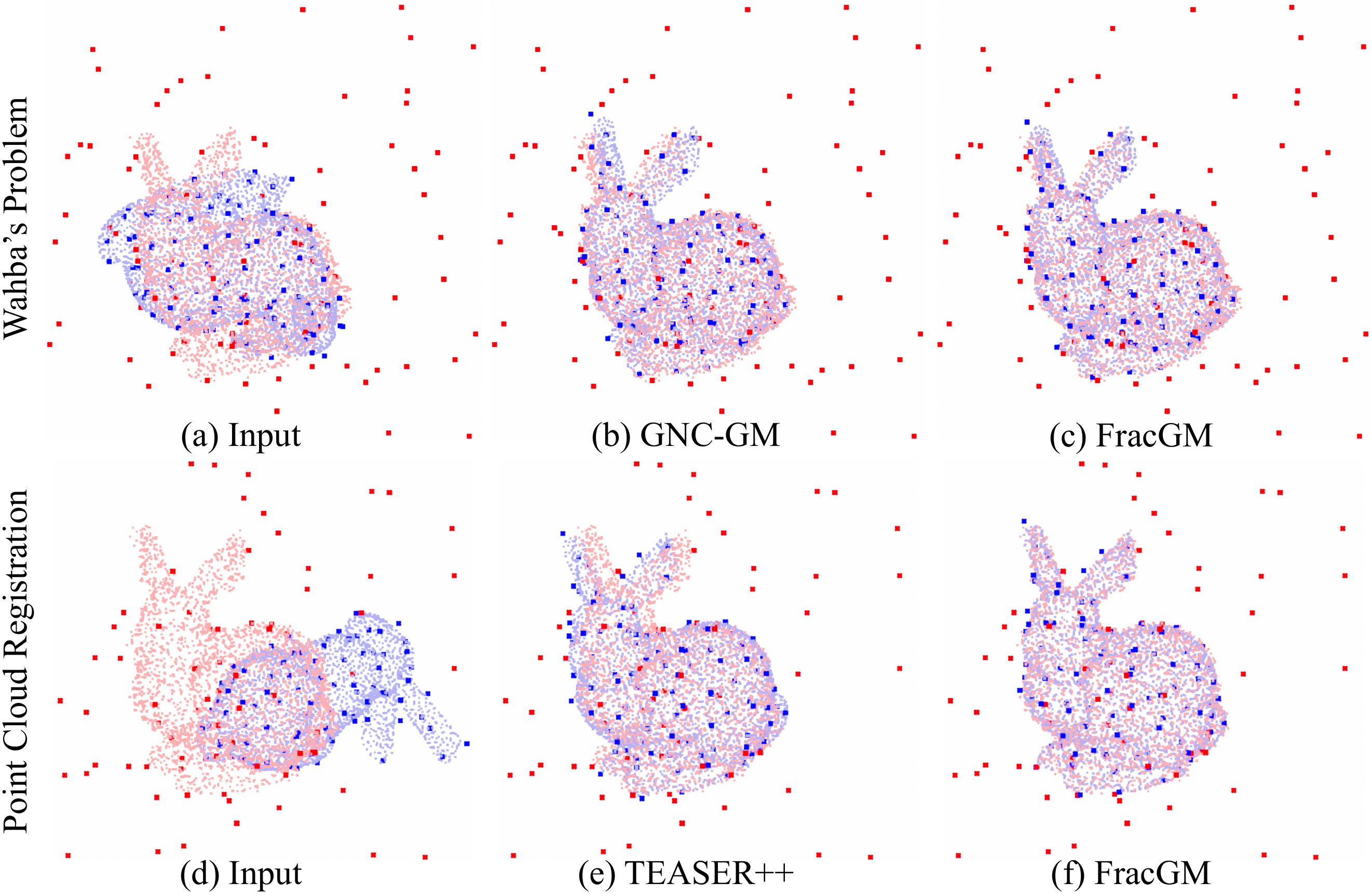}
    \caption{3-D registration of the Stanford Bunny dataset, with $N=100$ points of the source set in {\color{blue}blue} and target set in {\color{red}red} with 80\% outliers, while the light-colored points are only for visualization purpose. Our proposed FracGM for rotation estimation and point cloud registration is more accurate than state-of-the-art methods.} 
    \label{fig_bunny}
    \vspace{-0.6cm}
\end{figure}

Introducing robust functions into estimation tasks usually generates additional non-convexity to the corresponding optimization problem. This major drawback poses a challenge to the convergence and global optimality guarantees of robust estimators. One strategy aims to iteratively update the auxiliary parameters (mainly known as \textit{weights}) and solve the corresponding weighted least squares problem~\cite{BAG16, LWY19, YAT20, Z14}. It claims to be fast, accurate, and empirically practical for large-scale applications but lacks theoretical support for convergence and global optimality. In certain cases, trust-region methods are also required to stabilize the optimization process~\cite{FM17, RKL14}. Another strategy, especially for spatial perception, utilizes \textit{relaxation} techniques to obtain globally optimal solutions to relevant dual problems. Global optimality is guaranteed by convex optimization for dual problems, as well as the tightness between the original and dual problems~\cite{CC18, CP15, LDH19, YC19, YSC21}. However, most approaches rely on large semidefinite programming~(SDP), which results in computational inefficiency for large-scale applications. For example, the \textit{binary cloning} technique introduced in TEASER~\cite{YC19, YSC21} will increase the problem dimension of SDP by the number of measurements.

Designing an efficient robust estimator for real-world applications while justifying global optimality becomes essential and challenging.
Fortunately, we noticed that robust estimation with Geman-McClure~(GM) robust function is a special case of sum-of-ratio problems, in which its convergence and global optimality using fractional programming have been widely discussed~\cite{D67, J12, S97, SS03}.
To the best of our knowledge, most studies suppose that the numerators are convex and the denominators are concave.
This violates the assumption of the Geman-McClure robust function, as both numerators and denominators are typically convex.
To this end, we proposed a fast \textbf{Frac}tional programming technique for \textbf{G}eman-\textbf{M}cClure robust estimator (FracGM) by extending Jong's approach~\cite{J12}.
In particular, we prove that the corresponding dual problems of Geman-McClure robust estimation problems are convex during the iterative optimization process.
The benefits are threefold. First, it ensures that each iteration step of our solver is well-posed. Second, the final solutions of our solver correspond to the primary Geman-McClure robust estimation problem. Third, it allows us to study the conditions that make the proposed solver's solution globally optimal.
For the sake of completeness, we revisit Jong's approach~\cite{J12}, and establish theoretical support for the solver proposed in this paper.

We show the robustness of FracGM on Wahba's rotation problem\cite{W65} and 3-D point cloud registration, as shown in Figure~\ref{fig_bunny}. Both problems introduce an additional challenge for the non-convexity of the 3-D rotation group $\text{SO}(3)$. In this work, we simply relax the feasible space from $\text{SO}(3)$ to $\mathbb{R}^{3\times 3}$, and post-process the matrix with the orthogonal Procrustes problem. Despite breaking the statement of global optimality in FracGM, empirically the proposed method demonstrates better robustness and accuracy than most state-of-the-art approaches.
As a rotation solver, even with outlier rates above 90\%, FracGM achieves rotation errors under 1 degree approximately 80\% of the times, outperforming other solvers in a synthetic dataset. 
As a registration solver, when the outlier rates increase from 20\% to 80\%, the rotation and translation errors of FracGM increase by 53\% and 88\% less than those of the TEASER++\cite{YSC21} algorithm. In terms of registration in real-world scenes, FracGM achieves the best results in 13 of 18 outcomes, while having a 19.43\% improvement of computation time compared to TEASER++.
Overall, our main contributions are summarized as follows:
\begin{enumerate}{}{}
\item{We propose a fast Geman-McClure robust estimator with conditionally global optimality guarantees.}
\item{We provide FracGM-based rotation and registration solvers, outperforming existing state-of-the-art methods in both accuracy and robustness.}
\end{enumerate}

\backgroundsetup{contents={}}

\section{Related Work}\label{sec_related_work}

\subsection{Geman-McClure Robust Solvers}

Robust estimation problems are usually reformulated as iterative weighted least squares problems~\cite{BAG16, MB15}.
In Geman-McClure robust estimations, most approaches rely on Black-Rangarajan duality~\cite{BR96}, which is a weighted least squares with an additional regularization term~\cite{B19,YAT20,ZPK16}.
For example, Graduated Non-Convexity~(GNC)~\cite{YAT20} proposes a general-purpose robust solver for both truncated least squares (TLS) and Geman-McClure (GM) robust cost in conjunction with such duality technique, and introduces an approximation optimization process from convexity to non-convexity.
With recent advances of semidefinite relaxation~\cite{FDJ16}, there are also approaches to recast Geman-McClure robust functions to semidefinite programming~(SDP)~\cite{CP15}.

Unlike the robust estimators based on the Black-Rangarajan duality and semidefinite relaxation Geman-McClure, FracGM leverages Dinkelbach's transform~\cite{D67} and the Lagrangian to reformulate the problem as an underlying convex program with additional auxiliary variables.
It was originally introduced in the field of fractional programming optimization~\cite{D67, J12, S97, SS03}, while to the best of our knowledge, FracGM is the first approach to apply such technique for Geman-McClure robust estimation.

\subsection{Robust Estimation for Point Cloud Registration}
Point cloud registration is the process of aligning two point clouds by finding a spatial rigid transformation.
Horn~\cite{H87} first proposes a closed-form solution by decoupling the problem to scale, rotation and translation sub-problems.
With advances of relaxation techniques in recent decades, Olsson and Eriksson~\cite{OE08} apply Lagrangian duality with semidefinite relaxation.
In addition, Briales and Gonzalez-Jimenez~\cite{BG17} introduce orthonormality and determinant constraints such that the relaxation is empirically tight.
Those methods are said to be outlier-free approaches, since they assume that all measurement noises are modeled under ordinary Gaussian distributions.

In real-world applications, however, there are outlier measurements which may interfere with estimation accuracy.
For the purpose of mitigating the impact of outliers, robust estimation becomes a promising strategy in point cloud registration to identify and isolate bad measurements.
Zhou~\textit{et~al.}~\cite{ZPK16} propose the idea to solve robust registration by gradually adjusting the Geman-McClure robust function.
Yang~\textit{et~al.}~\cite{YSC21} develop Truncated least squares Estimation And SEmidefinite Relaxation (TEASER), where they reformulate the problem to large-scale SDP.
They also develop a fast implementation TEASER++ by leveraging GNC rotation solver to prevent large computational cost in SDP.

\section{Fractional Programming for Geman-McClure Robust Estimator}\label{sec_fracgm}

In this section, we provide the proposed fractional programming technique for Geman-McClure robust estimator along with theoretical optimal-guarantees of the solver. In general, a robust estimation problem has the following form~\cite{MB15, YAT20}:
\begin{equation}\label{general_form}
    \begin{aligned}
        \min_{\bx\in\mathbb{R}^d}\ &\sum_{i=1}^N \rho\big( r_i(\bx)\big).
    \end{aligned}
\end{equation}
Here $\rho(\cdot)$ is a robust function with $N$ residual functions $r_i(\cdot)$. Note that least squares problem corresponds to setting the robust function as $\rho(r)=r^2$. Applying the Geman-McClure function,
\begin{equation}\label{eq_gm}
    \rho_c(r)=\frac{c^2r^2}{r^2+c^2},
\end{equation}
Problem~\eqref{general_form} becomes
\begin{equation} \label{gm_general_form}
    \begin{aligned}
        \min_{\bx\in\mathbb{R}^d}\ &\sum_{i=1}^N \frac{c^2 r_i^2(\bx)}{r_i^2(\bx) + c^2}, \\
    \end{aligned}
\end{equation}
where $c > 0$ is a threshold that depicts the shape of robust function. We suppose that of all square residual functions $r_i^2(\cdot)$ are convex and twice continuously differentiable, and also suppose that we have $M$ additional convex constraints $g_j(\boldsymbol{x}),\ j=1,\dots,M$. In this way, Problem~\eqref{gm_general_form} is considered to be the following \textit{sum-of-ratio problem}~\cite{FJ01}:
\begin{equation}\label{eq_qsrp}
    \begin{aligned}
        \min_{\bx\in\mathbb{R}^d}\ &\sum_{i=1}^N\frac{f_i(\bx)}{h_i(\bx)}\\
        \text{s.t. }&g_j(\boldsymbol{x})\leq0,\ j=1,\dots,M,
    \end{aligned}
\end{equation}
in which $f_i(\bx) = c^2 r_i^2(\bx)$ and $h_i(\bx) = r_i^2(\bx) + c^2$ are also both convex and twice continuously differentiable.

We remark that each ratio's numerator and denominator in Jong's sum-of-ratio programming approach~\cite{J12} are required to be convex and concave, respectively.
However, Problem~\eqref{eq_qsrp} does not meet the assumption because both numerator and denominator are convex.
This discrepancy contradicts the theoretical foundation established in \cite{J12} for the existence and global optimality of solutions, which relies on the convexity of the linear combination of $f_i$ and $h_i$ with arbitrary non-negative coefficients.
Nevertheless, we claim that it is still possible to solve the program by introducing additional conditions, thanks to the structure of Geman-McClure function. To fill the theoretical gap, we revisit the method in the following sections, and provide conditions to guarantee theoretical optimality.

\subsection{Sum-of-Ratio Programming}\label{sec_sorp}

We write an equivalent problem for Problem~\eqref{eq_qsrp} by introducing upper bound variable $\bbeta\in\mathbb{R}^N$:
\begin{equation}\label{eq_primal}
    \begin{aligned}
        \min_{\bx\in\mathbb{R}^d}\ &\sum_{i=1}^N\bbeta_i\\
        \text{s.t. }&\frac{f_i(\bx)}{h_i(\bx)}\leq\bbeta_i,\ i=1,\dots,N\\ 
        &g_j(\bx)\leq0,\ j=1,\dots,M.
    \end{aligned}
\end{equation}

With Dinkelbach’s transform~\cite{D67} and Lagrangian, the following \thref{lemma_1} derives an underlying convex problem of the equivalent Problem~\eqref{eq_primal}. It also claims the relationship between the minimizer of Problem~\eqref{eq_qsrp} and newly introduced \textit{auxiliary variables} $\bbeta$ and $\bmu$.

\begin{lemma}[{{Variant of \cite[Lemma 1]{J12}}}]\thlabel{lemma_1}
If $(\bar{\bx},\bar{\bbeta})$ is a solution of Problem~\eqref{eq_primal}, then there exists $\bar{\bmu}\in\mathbb{R}^N$ such that $\bar{\bx}$ is a solution of the following problem for $\bmu=\bar{\bmu}$ and $\bbeta=\bar{\bbeta}$:
\begin{equation}\label{eq_dual}
    \begin{aligned}
        \min_{\bx\in\mathbb{R}^d}\ &\sum_{i=1}^N\bmu_i(f_i(\bx)-\bbeta_ih_i(\bx))\\
        \text{s.t. }&g_j(\bx)\leq0,\ j=1,\dots,M,
    \end{aligned}
\end{equation}
where $\bmu_i>0$ and $c^2>\bbeta_i$. Furthermore, $\bar{\bx}$ also satisfies the following system of equations for $\bmu=\bar{\bmu}$ and $\bbeta=\bar{\bbeta}$:
\begin{align}
    \bmu_i &=\frac{1}{h_i(\bx)},\ i=1,\dots,N, \label{eq_mu} \\
    \bbeta_i &=\frac{f_i(\bx)}{h_i(\bx)},\ i=1,\dots,N. \label{eq_beta}
\end{align}
\end{lemma}

\begin{proof}
We begin by providing a high-level overview of the proof. Firstly, $(\bar{\bx},\bar{\bbeta},\bar{\bmu})$ is claimed to be a KKT point of Problem~\eqref{eq_dual}. Secondly, Problem~\eqref{eq_dual} is shown to be a convex program. In this way, by the KKT sufficient conditions, $(\bar{\bx},\bar{\bbeta},\bar{\bmu})$ is an optimal solution of Problem~\eqref{eq_dual}. Please refer to~\cite{J12} for detailed description of the first part.

For the second part, we point out that $h(\bx)$ is convex in our case instead of concave as in~\cite{J12}. Hereby we introduce additional \textit{auxiliary variable constraints} $\bmu_i>0$ and $c^2>\bbeta_i$ in~\thref{lemma_1} in order to make the convexity of Problem~\eqref{eq_dual} hold. To obtain the auxiliary variable constraints, we expand the objective of Problem~\eqref{eq_dual}:
$$
\begin{aligned}
    \sum_{i=1}^N\bmu_i(f_i(\bx)-\bbeta_ih_i(\bx))
    &=\sum_{i=1}^N\bmu_i(c^2-\bbeta_i)r_i^2-\bmu_i\bbeta_ic^2,
\end{aligned}
$$
which implies that the objective of Problem~\eqref{eq_dual} is convex given $\bmu_i>0$ and $c^2>\bbeta_i$. 
\end{proof}

In this way, the primal problem~\eqref{eq_primal} turns out to be finding a feasible solution $\bar{\bx}$ of the dual problem~\eqref{eq_dual} as well as auxiliary variables $(\bar{\bbeta}, \bar{\bmu})$. We resort to \textit{alternating minimization} that iteratively and separately updates $\bx$ and $(\bbeta, \bmu)$. Specifically, in the $k$-th iteration, given a fixed $(\bbeta^k, \bmu^k)$ that satisfies the auxiliary variable constraints $\bmu^{k}_i>0$ and $c^2>\bbeta^{k}_i$, solving $\bx^{k}$ with Problem~\eqref{eq_dual} becomes a convex programming with fixed variable dimension $\mathbb{R}^d$. As for finding $(\bbeta^{k+1}, \bmu^{k+1})$ with given $\bx^k$, a root finding problem is introduced in the next section. It is same as Jong's approach~\cite{J12}, whereas we will additionally claim that the auxiliary variable constraints is theoretically preserved during the optimization.

\subsection{Solving Auxiliary Variables}\label{sec_aux}
Let $\balpha=(\bbeta^\top ,\bmu^\top )^\top \in\mathbb{R}^{2N}$ denote the auxiliary variables,
let $\bx_{\balpha}$ be the solution of problem \eqref{eq_dual} given an $\balpha$, and $r_i(\bx_{\balpha}),\ i=1,\dots,N$ are the corresponding residuals. Define a function $\psi:\mathbb{R}^{2N}\times\mathbb{R}^d\to\mathbb{R}^{2N}$ as follows:
$$
\psi(\balpha,\bx)=
\begin{pmatrix}
-f_1(\bx)+\bbeta_1h_1(\bx)\\
\vdots\\
-f_N(\bx)+\bbeta_Nh_N(\bx)\\
-1+\bmu_1h_1(\bx)\\
\vdots\\
-1+\bmu_Nh_N(\bx)
\end{pmatrix}.
$$

This function is specifically defined by the constraints \eqref{eq_mu} and \eqref{eq_beta} such that $\psi(\bar{\balpha},\bar{\bx})=\mathbf{0}$, where $\bar{\balpha}=(\bar{\bbeta}^\top,\bar{\bmu}^\top)^\top$ and $(\bar{\bx},\bar{\bbeta},\bar{\bmu})$ is the solution of Problem~\eqref{eq_dual}. 
By \cite[Corollary 2.1]{J12}, if $\bar{\bx}$ is a solution of Problem~\eqref{eq_qsrp}, then there exists $\balpha\in\mathbb{R}^{2N}$ such that $\bar{\bx}=\bx_{\balpha}$ and
\begin{equation}\label{eq_psi}
    \psi(\balpha,\bx_{\balpha})=\mathbf{0}.
\end{equation} 
This shows that a solution of the auxiliary variables $\balpha$ can be obtained by solving $\psi(\balpha,\bx_{\balpha})=\mathbf{0}$. 
\thref{thm} next states that such solution satisfies the KKT conditions in \thref{lemma_1}.

\begin{theorem}\thlabel{thm}
If $(\bar{\balpha},\bx_{\bar{\balpha}})$ is a solution for Problem~\eqref{eq_psi}, and let $\bar{\balpha}=(\bar{\bbeta}^\top ,\bar{\bmu}^\top)^\top$, then $\bar{\bmu}_i>0$ and $c^2>\bar{\bbeta}_i,\ i=1,\dots,N$. Furthermore, Problem~\eqref{eq_dual} has a solution where $(\bx,\bbeta,\bmu) = (\bx_{\bar{\balpha}},\bar{\bmu},\bar{\bbeta})$.
\end{theorem}

\begin{proof}
By \eqref{eq_mu}, $\bar{\bmu}_i>0$ since $h_i(\bx_{\bar{\balpha}})>0$. If $c^2\leq\bar{\bbeta}_i$, then
$$
c^2\leq\frac{c^2r_i(\bx_{\bar{\balpha}})^2}{r_i(\bx_{\bar{\balpha}})^2+c^2}
\implies
r_i(\bx_{\bar{\balpha}})^2+c^2\leq r_i(\bx_{\bar{\balpha}})^2.
$$
The above inequality yields $c^2\leq0$, which contradicts $c$'s definition. Therefore, we conclude $c^2>\bar{\bbeta}_i$. It follows from \thref{lemma_1} that $\bx_{\bar{\balpha}}$ is a solution of Problem~\eqref{eq_dual} which satisfies the KKT conditions.
\end{proof}

\thref{thm} provides an assurance that all $\bx_{\bar{\balpha}}$ such that $\psi(\bar{\balpha}, \bx_{\bar{\balpha}})=\mathbf{0}$ is a solution candidate for Problem~\eqref{eq_qsrp} that satisfies necessary optimality conditions, and also guarantees that we can indeed use Jong's approach~\cite{J12} to obtain a solution in our case where numerators and denominators of the fractional program are all convex.
To this end, given a fixed $\balpha^k=((\bbeta^k)^\top,(\bmu^k)^\top)^\top$ that satisfies the auxiliary variable constraints, we solve the convex program~\eqref{eq_dual} and obtain $\bx^k=\bx_{\balpha^k}$. If $\psi(\balpha^k,\bx^k)=\mathbf{0}$, by \thref{thm}, FracGM converges to a local solution of Problem~\eqref{eq_qsrp}. Otherwise, we update $\balpha$ by solving $\psi(\balpha,\bx^k)=\mathbf{0}$, where $\bx^k=\bx_{\balpha^k}$ is the fixed solution of Problem~\eqref{eq_dual} from the previous iteration in the sense of alternating minimization. Clearly, Problem~\eqref{eq_psi} is a linear system with fixed $\bx^k$, we write the \textit{closed-form} solution $\balpha^{k+1}$ component-wise as follows:
\begin{equation}
    \bbeta_i^{k+1}=\dfrac{f_i(\bx^k)}{h_i(\bx^k)},\ 
    \bmu_i^{k+1}=\dfrac{1}{h_i(\bx^k)}, \ i=1,\dots,N. \label{eq_newton}
\end{equation}

In summary, we iteratively update $\balpha^k$ by Equation~\eqref{eq_newton} and solve $\bx^k$ by Problem~\eqref{eq_dual} until the \textit{stopping criteria:} $\psi(\balpha^{k},\bx^k)=\mathbf{0}$. 
Note that during the $k^\prime$-th iteration that hasn't converge yet, $\balpha^{k^\prime}$ is the solution of $\psi(\balpha,\bx^{k^\prime-1})=\mathbf{0}$ from the previous iteration, however, not the solution of $\psi(\balpha^{k^\prime},\bx^{k^\prime})=\mathbf{0}$ in the current iteration. This stopping criteria can also be interpreted as finding an $\balpha^k$ such that both Problem~\eqref{eq_dual} and Problem~\eqref{eq_psi} are solved.
The complete procedure is depicted in Algorithm \ref{alg_fgm}.

\vspace{-0.2cm}
\setlength{\textfloatsep}{2pt}
\begin{algorithm}[htbp]
\caption{FracGM}\label{alg_fgm}
\begin{algorithmic}[1]
\Require $\bx^0$: initial guess, $r_i(\bx)$: residual functions
\Ensure $\bx^*$: estimation
\State compute $f_i,h_i$ by \eqref{eq_qsrp} and constraints $g_j$ if necessary
\While{iteration and tolerance condition}
    \State update $\bbeta^k$ and $\bmu^k$ by Equation~\eqref{eq_newton}
    \State update $\balpha^k\gets((\bbeta^k)^\top ,(\bmu^k)^\top )^\top $
    \State solve the convex programming \eqref{eq_dual}:
    \vspace{-0.12cm}
    $$
    \begin{aligned}
    \bx^k=\min_{\bx\in\mathbb{R}^d}\ &\sum_{i=1}^N\bmu_i^k(f_i(\bx)-\bbeta_i^kh_i(\bx))\\
    \text{s.t. }&g_j(\bx)\leq0,\ j=1,\dots,M\\
    \end{aligned}
    $$
    \vspace{-0.12cm}
    \If{$\psi(\balpha^k,\bx^k)=\mathbf{0}$} $\bx^*\gets \bx^k$ and break
    \EndIf
\EndWhile
\Return $\bx^*$
\end{algorithmic}
\end{algorithm}
\vspace{-0.4cm}

\subsection{Global Optimality of FracGM}\label{sec_global}

In Section~\ref{sec_sorp} and \ref{sec_aux}, we have shown that the solution of FracGM is a KKT point of Problem~\ref{eq_dual} by \thref{thm}, which is a local solution of Problem~\eqref{eq_primal} by \thref{lemma_1}. It follows from \cite[Corollary 2.1]{J12} that if Problem~\eqref{eq_psi} has only one solution, then such unique solution is the optimal solution of Problem~\eqref{eq_primal}. Consequently, we derive the following statement from the uniqueness of Problem~\eqref{eq_psi} {\cite[Theorem 3.1]{J12}}:

\begin{prop}\thlabel{prop}
Suppose that the global solution of Problem~\eqref{eq_qsrp} exists.
If $\psi(\balpha,\bx_{\balpha})$ is differentiable and Lipschitz continuous in $\mathbb{R}^{2N}$, then FracGM is a global solver.
\end{prop}

An example of FracGM satisfying~\thref{prop}, i.e., having theoretically global optimal guarantees, is given in the supplementary material.

\section{Spatial Perception with FracGM}\label{sec_fracgm_application}

In this section, we demonstrate two examples of applying the proposed FracGM estimator to practical problems, a FracGM-based rotation solver (Section~\ref{sec_rotation}) and a FracGM-based registration solver (Section~\ref{sec_fgm_reg}). Although only two examples are showed in this work, we reiterate that the proposed FracGM solver can apply to any problem in which the square of residuals are twice continuously differentiable and convex.

\subsection{Robust Rotation Solver with FracGM}\label{sec_rotation}

We first discuss about the Wahba's rotation Problem \cite{W65}, which is basically a registration of two point clouds that only differ by a 3-D rotation. Let $\{\boldsymbol{a}_i\},\{\boldsymbol{b}_i\}$ be two sets of point clouds with $N$ points, source set and target set, where $\boldsymbol{a}_i,\boldsymbol{b}_i\in\mathbb{R}^3$ are 3-D points. The goal is to find a rotation matrix $\mathbf{R}$ between the two point clouds, i.e., $\boldsymbol{b}_i=\mathbf{R}\boldsymbol{a}_i+\boldsymbol{\varepsilon}_i$, with error terms $\boldsymbol{\varepsilon}_i\sim\mathcal{N}(0,\sigma_i^2\mathbf{I}_3)$ following the isotropic Gaussian distribution. The maximum likelihood estimation is equivalent to the following least squares problem:
\begin{equation}\label{eq_rotation}
    \min_{\mathbf{R}\in\text{SO}(3)}\sum_{i=1}^N\frac{1}{\sigma_i^2}\big\|\boldsymbol{b}_i-\mathbf{R}\boldsymbol{a}_i\big\|^2_2.
\end{equation}

We first rewrite it as a quadratic program:
\begin{equation}\label{eq_rotation2}
    \text{Problem }\eqref{eq_rotation}
    \iff
    \begin{aligned}
        \min_{\boldsymbol{x}\in\mathbb{R}^{10}}\ &\sum_{i=1}^N\bx^\top (\frac{1}{\sigma_i^2}\mathbf{M}_i)\bx\\
        \text{s.t.}\ &\mathbf{R}\in\text{SO}(3) \text{ and } \boldsymbol{e}_{10}^\top \boldsymbol{x}=1
    \end{aligned}
\end{equation}
where $\mathbf{M}_i=[\boldsymbol{a}_i^\top \otimes\mathbf{I}_3,-\boldsymbol{b}]^\top [\boldsymbol{a}_i^\top \otimes \mathbf{I}_3,-\boldsymbol{b}]\in\mathbb{R}^{10\times10}$, $\boldsymbol{x}=[\text{vec}(\mathbf{R})^\top ,1]^\top \in\mathbb{R}^{10}$, $\text{vec}(\mathbf{R})\in\mathbb{R}^9$ is the column-wise vectorization of the rotation matrix $\mathbf{R}$, and $\boldsymbol{e}_{10}=(0,\dots,0,1)^\top\in\mathbb{R}^{10}$. This vectorization and homogenization technique is commonly used~\cite{BG17, OE08}.
We also remark that all square of residual functions $r_i^2(\bx) = \bx^\top (\frac{1}{\sigma_i^2}\mathbf{M}_i)\bx$ defined in Problem~\eqref{eq_rotation2} are twice continuously differentiable and \textit{quadratic} convex. Next, we apply the GM robust function~\eqref{eq_gm} and form the following optimization problem:
\begin{equation}\label{eq_gm_rotation}
    \begin{aligned}
        \min_{\bx\in\mathbb{R}^{10}}\ &\sum^N_{i=1}\frac{c^2\bx^\top (\frac{1}{\sigma_i^2}\mathbf{M}_i)\bx}{\bx^\top(\frac{1}{\sigma_i^2}\mathbf{M}_i)\bx+c^2}\\
        \text{s.t.}\ &\mathbf{R}\in\text{SO}(3) \text{ and } \boldsymbol{e}_{10}^\top \bx=1
    \end{aligned}
\end{equation}

In order to apply FracGM, we relax the $\text{SO}(3)$ constraints to align with the form of Problem~\eqref{eq_qsrp}. Since the numerators and denominators in the relaxed Problem \eqref{eq_gm_rotation} are quadratic,
the derived Problem~\eqref{eq_dual} is a quadratic program with only one equality constraint. Such problem has a \textit{closed-form} solution by Shur complement and the KKT matrix \cite{BV04}:
\begin{equation}\label{eq_closed}
    \bar{\bx}=\frac{\mathbf{A}^{-1}\boldsymbol{e}_{10}}{\boldsymbol{e}_{10}^\top\mathbf{A}^{-1}\boldsymbol{e}_{10}},\
\mathbf{A}=\sum_{i=1}^N\bmu_i(c^2-\bbeta_i)(\frac{1}{\sigma_i^2}\mathbf{M}_i).
\end{equation}

After finding a solution $\bx^*$ by FracGM, we dehomogenize $\bar{\boldsymbol{x}}\in\mathbb{R}^{10}$ to $\bar{\bx}^\prime\in\mathbb{R}^9$, then reshape it to a matrix $\mathbf{R}^\prime\in\mathbb{R}^{3\times3}$. Since the solution is on the relaxed convex space and we are agnostic on theoretical conditions of tightness, we forcibly project it back to the special orthogonal space by SVD \cite{S66}. This technique is also commonly used after linear relaxation in rotation estimation~\cite{BG17, OE08}. 
In light of an initial guess is needed for FracGM, we simply use the fast closed-form solution \cite{H87} for Problem~\eqref{eq_rotation}. Despite that this closed-form solution does not tolerate outliers, we will show that it is proficient even when there are numerous outliers in Section \ref{sec_experiments}. We summarize the FracGM-based rotation solver in Algorithm~\ref{alg_rotation}.

\vspace{-0.2cm}
\setlength{\textfloatsep}{0pt}
\begin{algorithm}[htbp]
\caption{FracGM-based rotation solver}\label{alg_rotation}
\begin{algorithmic}[1]
\Require $\mathbf{R}^0$: initial guess, $\boldsymbol{a}_i,\boldsymbol{b}_i$: point clouds
\Ensure $\mathbf{R}^*$: estimated rotation
\State rewrite $\mathbf{R}^0\in\text{SO}(3)$ to $\bx^0\in\mathbb{R}^{10}$
\State compute the residual functions $r_i(\bx)$ by \eqref{eq_rotation2}
\State compute $\bx^*\in\mathbb{R}^{10}$ by Algorithm 1
\State rewrite $\bx^*\in\mathbb{R}^{10}$ to $\mathbf{R}^\prime\in\mathbb{R}^{3\times3}$
\State project $\mathbf{R}^\prime\in\mathbb{R}^{3\times3}$ to $\mathbf{R}^*\in\text{SO}(3)$ by SVD\\
\Return $\mathbf{R}^*$
\end{algorithmic}
\end{algorithm}
\vspace{-0.4cm}

\subsection{Robust Registration Solver with FracGM}\label{sec_fgm_reg}
We extend the FracGM-based rotation solver to a registration solver, which is to find the spatial transformation between the two point clouds, i.e., $\boldsymbol{b}_i=\mathbf{R}\boldsymbol{a}_i+\boldsymbol{t}+\boldsymbol{\varepsilon}_i$, where $\mathbf{R}$ is the rotation matrix, $\boldsymbol{t}$ is the translation vector, and error terms $\boldsymbol{\varepsilon}_i\sim\mathcal{N}(0,\sigma_i^2\mathbf{I}_3)$ follow the isotropic Gaussian distribution. In this way, we get the following problem:
\begin{equation}\label{eq_registration}
    \min_{\mathbf{R}\in\text{SO}(3),\boldsymbol{t}\in\mathbb{R}^3}\sum_{i=1}^N\frac{1}{\sigma_i^2}\big\|\boldsymbol{b}_i-(\mathbf{R}\boldsymbol{a}_i+\boldsymbol{t})\big\|^2_2
\end{equation}

We extend Problem~\eqref{eq_gm_rotation} to solving spatial transformation and the problem formulation of GM robust point cloud registration becomes:
\begin{equation}\label{eq_gm_registration}
    \begin{aligned}
        \min_{\bx\in\mathbb{R}^{13}}\ &\sum^N_{i=1}\frac{c^2\bx^\top (\frac{1}{\sigma_i^2}\mathbf{M}_i)\bx}{\bx^\top (\frac{1}{\sigma_i^2}\mathbf{M}_i)\bx+c^2}\\
        \text{s.t.}\ &\mathbf{R}\in\text{SO}(3) \text{ and } \boldsymbol{e}^\top_{13} \bx=1
    \end{aligned}
\end{equation}
where $\mathbf{M}_i=[\boldsymbol{a}_i^\top \otimes \mathbf{I}_3,\mathbf{I}_3,-\boldsymbol{b}]^\top [\boldsymbol{a}_i^\top \otimes \mathbf{I}_3,\mathbf{I}_3,-\boldsymbol{b}]\in\mathbb{R}^{13\times13}$, $\bx=[\text{vec}(\mathbf{R})^\top ,\boldsymbol{t}^\top ,1]^\top \in\mathbb{R}^{13}$. The structural consistency between Equation~\eqref{eq_gm_rotation} and Equation~\eqref{eq_gm_registration} permits a similar algorithm to solve the registration problem. In particular, the algorithm of FracGM for point cloud registration can be extended from Algorithm \ref{alg_rotation} by adding an initial guess $(\mathbf{R}^0, \boldsymbol{t}^0)$ and solving $\bx^*\in \mathbb{R}^{13}$ in Problem~\eqref{eq_gm_registration}.

It is worth noting that the introduced FracGM-based registration solver directly estimates the relative transformation between two point clouds, whereas TEASER~\cite{YSC21} decouples rotation and translation estimation. This strategy is more computationally efficient, as it avoids the increase in measurement number complexity up to $O(N^2)$ when introducing translation-invariant measurements for rotation solvers. In addition, this computational cost-friendly strategy also improves the accuracy of estimation results. 

We acknowledge that the theoretical guarantee of global optimality for FracGM-based rotation and registration solvers remains an open issue. Specifically, the global optimality is subjected to variable space relaxation tightness and~\thref{prop} for the relaxed program, which could be challenging to verify as it depends on the input data. It requires further investigation to address the issue\footnote{Empirical studies on global optimality including relaxation tightness and sensitivity to initial guesses are available at the supplementary material.}.

\section{Experiments and Applications}\label{sec_experiments}
We evaluate the proposed FracGM solver in rotation estimation and point cloud registration with both synthetic dataset and realistic 3DMatch dataset~\cite{ZSN17}. In the synthetic dataset, we experiment the Stanford Bunny point cloud~\cite{GM94}, manipulate various proportions of points as outliers, and transform it to the target point cloud randomly. Hereby, point correspondences are predefined but noisy. This dataset validates methods' robustness to outlier point correspondences at different levels. Meanwhile, the 3DMatch dataset~\cite{ZSN17} is used to examine the real-world capability of using FracGM in point cloud registration. Point correspondences in this case are obtained by point cloud feature extraction and matching, and therefore the actual noise model and outlier rates are unknown. All experiments are conducted over 40 Monte Carlo runs, and computed on a desktop computer with an Intel i9-13900 CPU and 128GB RAM.

We compare FracGM with various state-of-the-art approaches. In rotation estimation, outlier-free approaches (SVD~\cite{H87}, RCQP~\cite{BG17}) and outlier-awareness approaches (GNC-TLS~\cite{YAT20}, GNC-GM~\cite{YAT20}, TEASER~\cite{YSC21}) are chosen to study the impact of different outlier distributions on performance. In point cloud registration, Fast Point Feature Histograms (FPFH)~\cite{FPFH} are used to extract point features and to perform feature matching with nearest neighbor searching. Outlier-awareness registration solvers (RANSAC~\cite{RANSAC}, FGR~\cite{ZPK16}, TEASER++~\cite{YSC21}) are chosen to evaluate registration accuracy and computational efficiency.

\subsection{Rotation Estimation with Synthetic Dataset}\label{rotation_estimation}

We evaluate the resistance to outliers in different outlier rates of the proposed FracGM rotation solver. In practice, we down sample the Stanford Bunny point cloud to $50$ points for outlier rates between 20\% and 80\%, and to $500$ points for outlier rates higher than 80\%. We construct the corresponding point cloud by adding a Gaussian noise with standard deviation of 0.01, and applying random rotation $\mathbf{R}\in\text{SO}(3)$. We generate outliers by replacing a fraction (outlier rate) of points by random sampled 3-D points inside the sphere of radius 2. The threshold $c$ in robust functions (TLS and GM) are all set to 1.

\begin{figure*}[htbp]
    \centering
    \begin{subfigure}{\textwidth}
        \centering
        \includegraphics[width=0.98\textwidth]{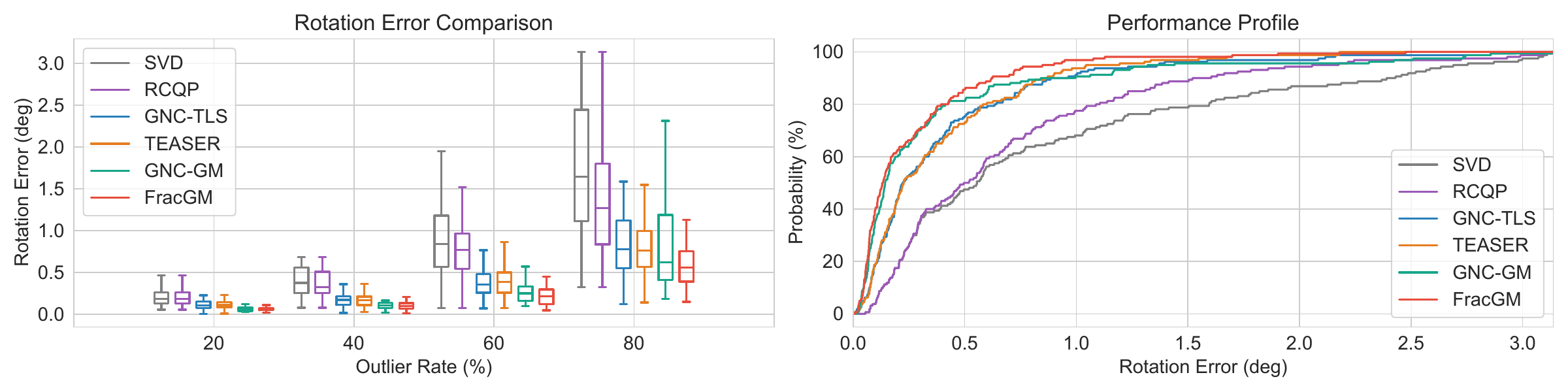}
        \vspace{-0.3cm}
        \caption{Rotation error comparison with outlier rates from 20-80\%.}
    \end{subfigure}

    \begin{subfigure}{\textwidth}
        \centering
        \includegraphics[width=0.98\textwidth]{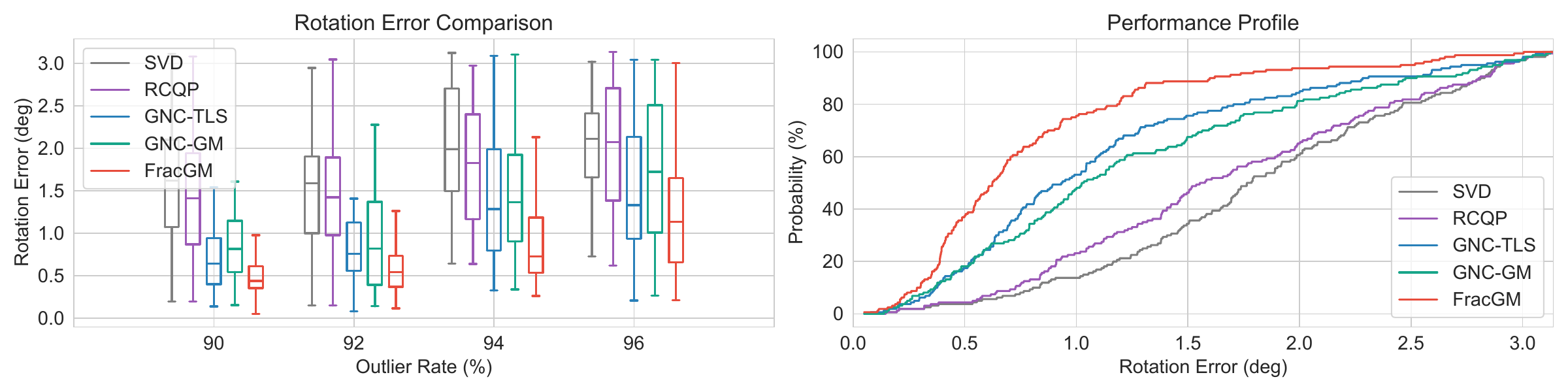}
        \vspace{-0.3cm}
        \caption{Rotation error comparison with more extreme outlier rates above 90\%.}
    \end{subfigure}
    
    \vspace{-0.1cm}

    \caption{Rotation error and Performance profile comparison with 
    (\romannum{1}) SVD~\cite{H87}, 
    (\romannum{2}) RCQP~\cite{BG17}, 
    (\romannum{3}) GNC-TLS~\cite{YAT20}, 
    (\romannum{4}) TEASER~\cite{YSC21}, 
    (\romannum{5}) GNC-GM~\cite{YAT20}, and 
    (\romannum{6}) FracGM.
    Note that TEASER cannot solve large size problems, thus it is not shown in the extreme outlier cases. 
    }
    \label{fig_rotation}
    \vspace{-0.6cm}
\end{figure*}

Figure \ref{fig_rotation} shows the box plot of error comparison and performance profile to evaluate rotation solvers. For outlier rates between 20\% and 80\%, FracGM has the lowest mean rotation error of $0.26^\circ$ and the highest probability of being the optimal solver than both outlier-free solvers and outlier-awareness solvers. FracGM stands out in extreme cases where the outlier rate is greater than 90\%.
In contrast to other solvers that can only limit up to 60\% of rotation errors within 1 degree, FracGM confines approximately 80\% of rotation errors within this range.

We also compare the error curve with TLS- and GM-based GNC~\cite{YAT20} in Figure~\ref{fig_converge}.
It reveals that FracGM not only has a better convergence speed than GNC approaches, but also a lower error.
We believe that GNC approximates the non-convex objective by some surrogate functions, in which such approximation might be far from the non-convex objective in the early iterations.
On the other hand, our solver reformulates the non-convex objective to an equivalent convex problem, thus we are solving the exact same objective in every iteration.
Thus, FracGM has a better ability to deal with outliers and can obtain a higher-quality estimation.

\vspace{-0.2cm}
\begin{figure}[htbp]
    \centering
    \includegraphics[width=\columnwidth]{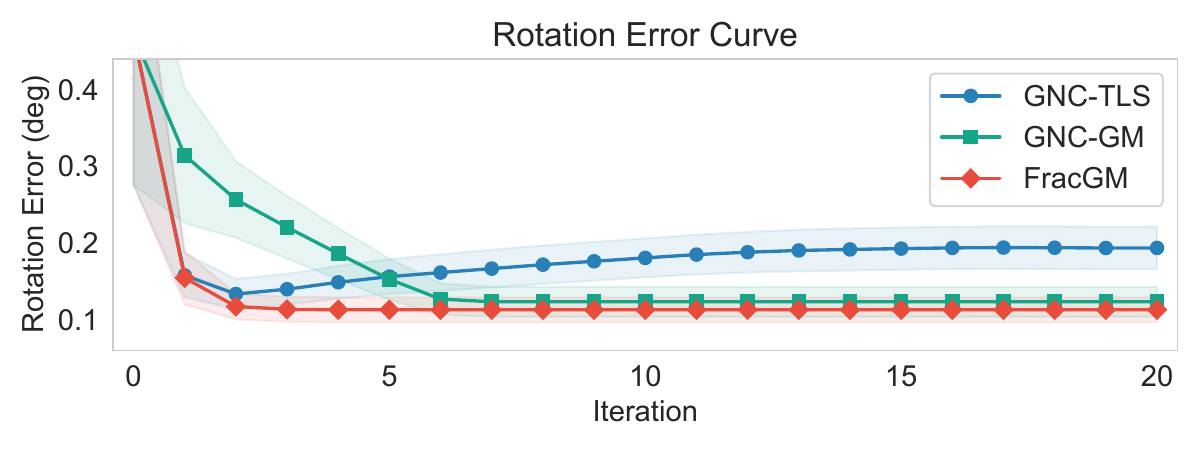}
    \vspace{-0.6cm}
    \caption{Convergence comparison between state-of-the-art iterative methods.}
    \label{fig_converge}
\end{figure}

\subsection{Point Cloud Registration with Synthetic Dataset}\label{sec_exp_reg}

We evaluate the proposed FracGM registration solver under two different scenarios.
In the first scenario, we down sample the Stanford Bunny point cloud to 500 points and manipulate different rates of outlier point correspondences from 20\% to 80\%, which is the same configuration of Section~\ref{rotation_estimation} with additional random translation $\boldsymbol{t}\in\mathbb{R}^3,\|\boldsymbol{t}\|\leq1$.
We set the noise bound \cite{YSC21} of TEASER++ and FracGM ($\sigma_i$ in Equation~\eqref{eq_gm_registration}) to 0.1, while RANSAC and FGR does not consider $\sigma_i$ in their problem formulation.
In the second scenario, we solve the problem given different noise bound guess to evaluate the impact of it. We believe that specifying the above two configurations are important in the registration task, since the error modeling of outlier in real-world applications is usually unknown and diverse.

We notice that FGR first filters bad correspondence by Fast Point Feature Histogram (FPFH) and nearest neighbors search, while TEASER++ has a Maximum Clique Inlier Selection (MCIS) step to prune outliers before estimating.
Both techniques serve as correspondence preprocessing, and since the main idea of \textit{robust estimator} is to deal with outliers by the solver itself, we exclude both steps in order to demonstrate the robustness of each solver.

\begin{figure*}[htbp]
    \centering
    \includegraphics[width=0.96\textwidth]{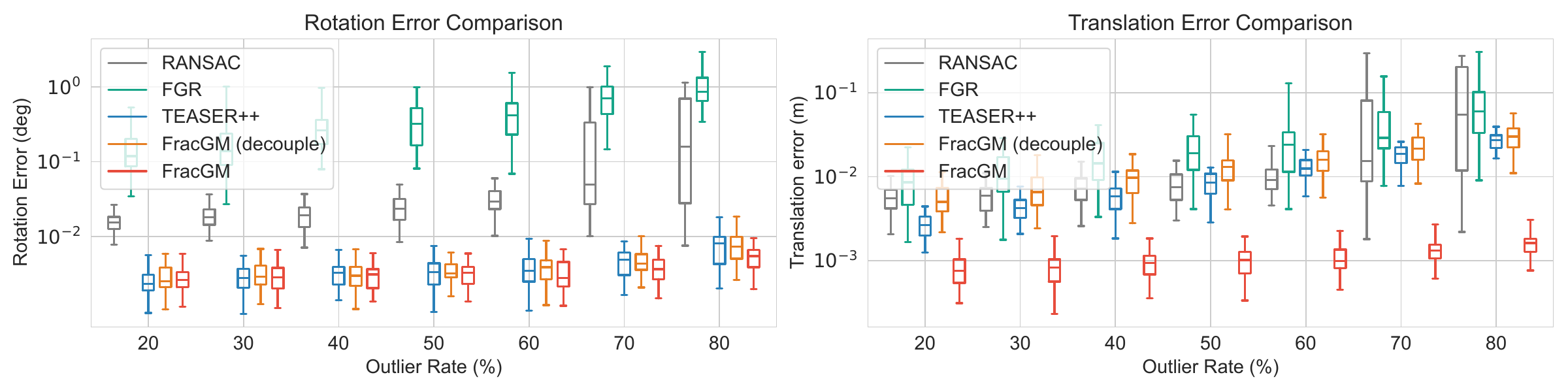}
    \vspace{-0.3cm}
    \caption{Performance comparison between state-of-the-art robust registration solvers in the synthetic dataset. FracGM is comparable to or better than other methods in terms of rotation and translation errors.}
    \label{fig_registration}
    \vspace{-0.36cm}
\end{figure*}

\begin{figure*}[htbp]
    \centering
    \includegraphics[width=0.96\textwidth]{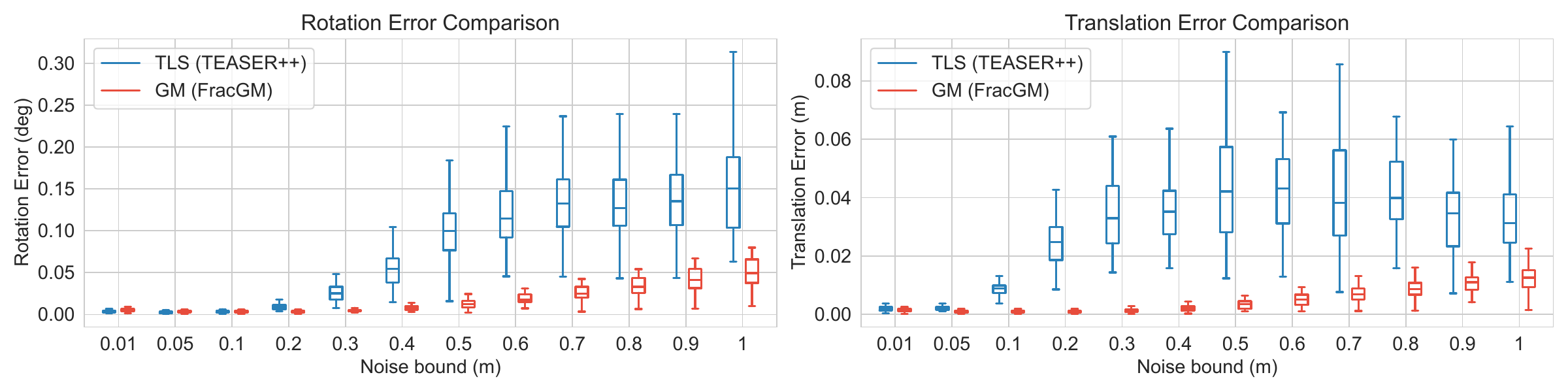}
    \vspace{-0.3cm}
    \caption{Performance comparison between TEASER++~\cite{YSC21} and FracGM registration solver with different noise bounds for the Stanford Bunny point cloud. Our solver is more insensitive to noise configuration with respect to error models, and sustains more accurate results than TLS-based approach.}
    \label{fig_sigma}
    \vspace{-0.6cm}
\end{figure*}

Figure~\ref{fig_bunny} demonstrates an example of this experiment. We run RANSAC for 100 times, which is a multi-times non-robust solver, while others are one-time robust solvers. Figure \ref{fig_registration} reports that FracGM is the only one that is truly robust to outliers and can really work well with larger outlier rate without any correspondence preprocessing or outlier pruning. 
As the outlier rate increased from 20\% to 80\%, the rotation and translation error of TEASER++ escalated by 186\% and 900\%, while FracGM increased only by 87\% and 106\%, representing an improvement of 53\% and 88\%, respectively.
As stated in Section~\ref{sec_fgm_reg}, we also extended a FracGM translation solver to adapt the decoupled structure of TEASER++. Figure~\ref{fig_registration} shows that the results of rotation estimation in FracGM (decouple) is similar to FracGM, while the accuracy of translation estimation in FracGM (decouple) falls. This is because the measurements for the decoupled translation solver are affected by the prior rotation estimation, and fixing a rotation matrix may cause the translation solver converge to some local solution. Table~\ref{table_time} reports that the computation time of FracGM (decouple) increases exponentially due to the fact that the measurements for the decoupled rotation solver also increases exponentially, while FracGM can deal with thousands of measurements with milliseconds.

\begin{table}[htbp]
    \vspace{-0.2cm}
    \centering
    \caption{Time comparison between FracGM and FracGM (decouple).}
    \begin{adjustbox}{width=0.48\textwidth}
    \begin{tabular}{lccccc}
    \\[-1.2em]\Xhline{2\arrayrulewidth}\\[-1.em]
    & \multicolumn{5}{c}{Average Computation Time (sec)}
    \\
    $N$ & 100 & 500 & 1000 & 2000 & 5000
    \\
    \\[-1.2em]\hline\\[-1.em]
    FracGM (decouple)
    & 0.0108 & 0.3891 & 1.6401 & 6.8731 & 46.1361\\
    FracGM
    & 0.0003 & 0.0012 & 0.0024 & 0.0053 & 0.0137\\
    \\[-1.2em]\Xhline{2\arrayrulewidth}
    \end{tabular}
    \end{adjustbox}
    \label{table_time}
    \vspace{-0.4cm}
\end{table}

In the above experiments, we choose a noise bound (the upper bound of $\sigma_i$ in Equation~\eqref{eq_gm_registration}) of 0.1 that is larger than the noise standard deviation of the data set at $0.01$. However, we may not know the actual noise of the measurements in reality, and thus we tend to set a larger noise bound guess. Figure~\ref{fig_sigma} reports the result of setting different noise bounds on a fixed dataset with 50\% of outliers, while the other settings are the same as before. It could be expected that the estimation will be more accurate if the noise bound is closer to the actual noise standard deviation. Despite this, the estimation error of TEASER++ increases dramatically.
For example, given a noise bound of 1 that is 100 times the noise standard deviation, FracGM has a 66\% and 62\% improvement over TEASER++ in terms of rotation and translation error.

\subsection{Point Cloud Registration with 3DMatch Dataset}\label{sec_exp_reg_3dmatch}

We test our proposed FracGM registration solver on the realistic 3DMatch dataset \cite{ZSN17} to examine the real-world application capability. 
We downsample point clouds with a voxel size of 5 cm, obtain the correspondence by FPFH, perform a nearest-neighbor search to generate sets of point correspondences and filter them with MCIS~\cite{YSC21}. 

Table~\ref{table3dmatch} exhibits the registration accuracy and computational time between our solver and other state-of-the-art methods (RANSAC~\cite{RANSAC}, FGR~\cite{ZPK16} and TEASER++~\cite{YSC21}). 
FracGM has the lowest average rotation error within $1.22^{\circ}$ and an average translation error less than 1.7 m. 
Since our solver converges fast and we only solve a fixed low-dimensional quadratic program and a linear equation system during each iteration, FracGM can be executed in real-time with an average of 25 milliseconds each registration.
In terms of rotation and translation errors across 9 tested scenes, FracGM achieves the best results in 13 out of 18 outcomes, 
while having a 19.43\% improvement of computation time compared to TEASER++.

\begin{table*}[t]
    \centering
    \caption{Performance comparison between RANSAC, FGR, TEASER++, and FracGM on the 3DMatch dataset. Note that RANSAC runs over 10 thousand times for each registration, and noise bounds of TEASER++ and FracGM are set to 0.1m.}
    \vspace{-0.1cm}
    \begin{adjustbox}{width=1\textwidth}
    \begin{tabular}{llcccccccccc}
    \\[-1.2em]\Xhline{2\arrayrulewidth}\\[-1.em]
    Scenes & Errors 
    & Kitchen & Home 1 & Home 2  & Hotel 1 & Hotel 2 & Hotel 3 & Study  & MIT Lab & Average & Time (s)\\
    \\[-1.2em]\hline\\[-1.em]
    \multirow{2}{*}{RANSAC-10K} 
    & Rotation (deg)  &    1.248 &    1.163 &    1.375 &    1.183 &    1.525 &    1.393 &    1.107 &    1.027 &    1.253 & \multirow{2}{*}{0.0151}\\
    & Translation (m) &    2.053 &    1.723 &    2.180 &    1.984 &    1.870 &    0.991 &    1.770 &    1.753 &    1.791 & \\
    \\[-1.2em]\hline\\[-1.em]
    \multirow{2}{*}{FGR \cite{ZPK16}}  
    & Rotation (deg)  &    1.275 &    1.206 &\bf{1.133}&    1.232 &    1.555 &    1.491 &    1.203 &    1.065 &    1.270 & \multirow{2}{*}{0.0373}\\
    & Translation (m) &    2.041 &    1.767 &\bf{2.007}&    2.068 &    2.031 &    1.224 &    1.911 &    1.888 &    1.867 & \\
    \\[-1.2em]\hline\\[-1.em]
    \multirow{2}{*}{TEASER++ \cite{YSC21}} 
    & Rotation (deg)  &    1.225 &\bf{1.146}&    1.357 &    1.177 &    1.525 &\bf{1.349}&    1.034 &    0.960 &    1.221 & \multirow{2}{*}{0.0313}\\
    & Translation (m) &    2.044 &    1.691 &    2.131 &    1.984 &    1.864 &    1.014 &    1.674 &\bf{1.575}&    1.747 & \\
    \\[-1.2em]\hline\\[-1.em]
    \multirow{2}{*}{FracGM} 
    & Rotation (deg)  &\bf{1.219}&    1.154 &    1.365 &\bf{1.168}&\bf{1.463}&    1.368 &\bf{1.021}&\bf{0.959}&\bf{1.215}& \multirow{2}{*}{0.0252}\\
    & Translation (m) &\bf{1.989}&\bf{1.629}&    2.134 &\bf{1.935}&\bf{1.745}&\bf{0.926}&\bf{1.566}&    1.592 &\bf{1.689}& \\
    \\[-1.2em]\Xhline{2\arrayrulewidth}
    \end{tabular}
    \end{adjustbox}
    \label{table3dmatch}
    \vspace{-0.6cm}
\end{table*}

\section{Conclusion}\label{sec_conclusion}
In this paper, we introduce FracGM, a fast algorithm for Geman-McClure robust estimation.
We revisit and revise the theoretical part of existing fractional programming techniques to accommodate the case of using Geman-McClure robust function.
In particular, we prove that the convexity of the dual problem is guaranteed during the iteration process.
It not only establishes the validity of our solver, but also allows us to study the global optimality conditions.
We demonstrate two FracGM-based solvers for real-world spatial perception problems, and empirical studies indicate that our solvers presents strong outlier rejection capability, good resilience to different error distributions, and fast computation speed.
In the future, we would like to investigate theoretical conditions that guarantee global optimality of FracGM-based solvers on real-world applications.

\section*{Supplementary Material}

Code and complementary documents are available at 
\url{https://github.com/StephLin/FracGM}.

\balance

\bibliographystyle{IEEEtran}

\end{document}